%% file: main.tex
\documentclass[letterpaper, 10 pt, conference]{ieeeconf}
\usepackage{balance}
\IEEEoverridecommandlockouts                    
\input{preamble.tex}

\title{\LARGE \textbf{Dynamically Feasible Path Planning in Cluttered Environments \\ via Reachable B\'ezier Polytopes}}
\begin{document}

\author{Noel Csomay-Shanklin, William D. Compton, Aaron D. Ames
\thanks{The authors are with the Control and Dynamical Systems Department at the California Institute of Technology. Email: \{\texttt{noelcs@caltech.edu}\}.}%
\thanks{This research is supported by Technology Innovation Institute (TII).}
%\thanks{This research was supported by Technology Innovation Institute (TII), NSF Graduate Research Fellowship No. DGE‐1745301, AeroVironment, NSF Grant No. 1918655 and Raytheon, Beyond Limits, JPL RTD 1643049.}
}

\maketitle

\begin{abstract}%

The deployment of robotic systems in real world environments requires the ability to quickly produce paths through cluttered, non-convex spaces. These planned trajectories must be both kinematically feasible (i.e., collision free) and dynamically feasible (i.e., satisfy the underlying system dynamics), necessitating a consideration of both the free space and the dynamics of the robot in the path planning phase. 
In this work, we explore the application of \textit{reachable B\'{e}zier polytopes} as an efficient tool for generating trajectories satisfying both kinematic and dynamic requirements. 
Furthermore, we demonstrate that by offloading specific computation tasks to the GPU, such an algorithm can meet tight real time requirements.
We propose a layered control architecture that efficiently produces collision free and dynamically feasible paths for nonlinear control systems, and demonstrate the framework on the tasks of 3D hopping in a cluttered environment.
\end{abstract}

\section{Introduction}

The development of control strategies which enable robots to operate in cluttered environments has been a focal challenge in control, dating back to the early days of mobile robots \cite{nilsson_mobile_1969}.
While it is theoretically possible to accomplish this task in a single shot 
via policy-design \cite{wang_survey_2021, qureshi_motion_2020, pfeiffer_perception_2017, bency_neural_2019}, this poses significant challenges including reasoning about high dimensional state and action spaces, nonlinear dynamics, and having a lack of safety and feasibility guarantees. 
As a result, it is often more practical to deploy hierarchical control architectures, whereby high level path planning algorithms are paired with low-level tracking controllers. Typically, kinematically feasible (i.e., collision free) paths are passed down the control stack, and dynamic feasibility (i.e., satisfaction of the underlying system dynamics) is an assumed property of the low-level controller. 
This top-down approach, however, provides minimal guarantees of state and input constraint satisfaction, as violations of this assumption at the low level may cause significant deviations from the planned high-level trajectory.

% the existence of a feedback controller such that the resulting trajectory satisfies the system's closed loop dynamics

% From a first principles perspective, tackling this problem purely at either the low-level or the high-level results in infeasibilities.

% Constructing a pipeline that guarantees feasibility requires the production of paths which are \textit{kinodynamically feasible}, meaning they are both kinematically feasible (collision free) as well as dynamically feasible (satisfy the underlying system dynamics).
% In addressing the challenge of producing kinodynamically feasible trajectories, 

% This has been tackled both from a dynamics-first perspective \cite{ames2019control}, as well as a kinematic-first perspective \cite{}.

% The production of collision free paths in cluttered environments has been a focal challenge in control since the early days of mobile robots \cite{nilsson_mobile_1969} due to its direct applicability in helping robots enter the real world. 

% Hierarchical control architectures are a common method for tackling this problem, whereby high level path planning algorithms are paired with low-level tracking controllers. 
%

Constructing a pipeline which is able to guarantee feasibility necessitates reasoning about both the dynamics and free space information during the planning phase; algorithms which plan for both kinematic and dynamic feasibility are termed \textit{kinodynamic planners} \cite{schmerling2021kinodynamic}.
Kinodynamic planners can broadly be placed into two categories: sampling-based and optimization-based methods.
%
%
%
% Both of these methodologies produce a graph, either of states or of sets, whereby the problem of path generation is reduced to a graph search. 
% From the sampling-based side, a lattice-based kinodynamic planner was proposed in \cite{donald1993kinodynamic} which developed a planner over discrete policy motion primitives to generate subsequent states.
%
%
%
In the sampling-based paradigm, one main way of producing kinodynamically feasible paths is by sampling policies from a discrete collection of predefined primitives \cite{donald1993kinodynamic, lavalle2001randomized}. This method suffers from the fact that the predefined policies induce a bias and may not make meaningful progress towards the goal. Another approach is to randomly sample new points and connect them the nearest node on the graph if a two point boundary value problem is feasible \cite{webb_kinodynamic_2013}, which can be viewed as an extension to the classic RRT \cite{lavalle1998rapidlyexploring} method of kinematic path planning. This method requires solving as many such problems as there are nodes in the graph, which can be expensive to compute for high dimensional nonlinear systems.

\begin{figure}[t!]
    \centering
    \href{https://www.youtube.com/watch?v=TrScjfhp3G4}{
    \includegraphics[width=0.95\columnwidth]{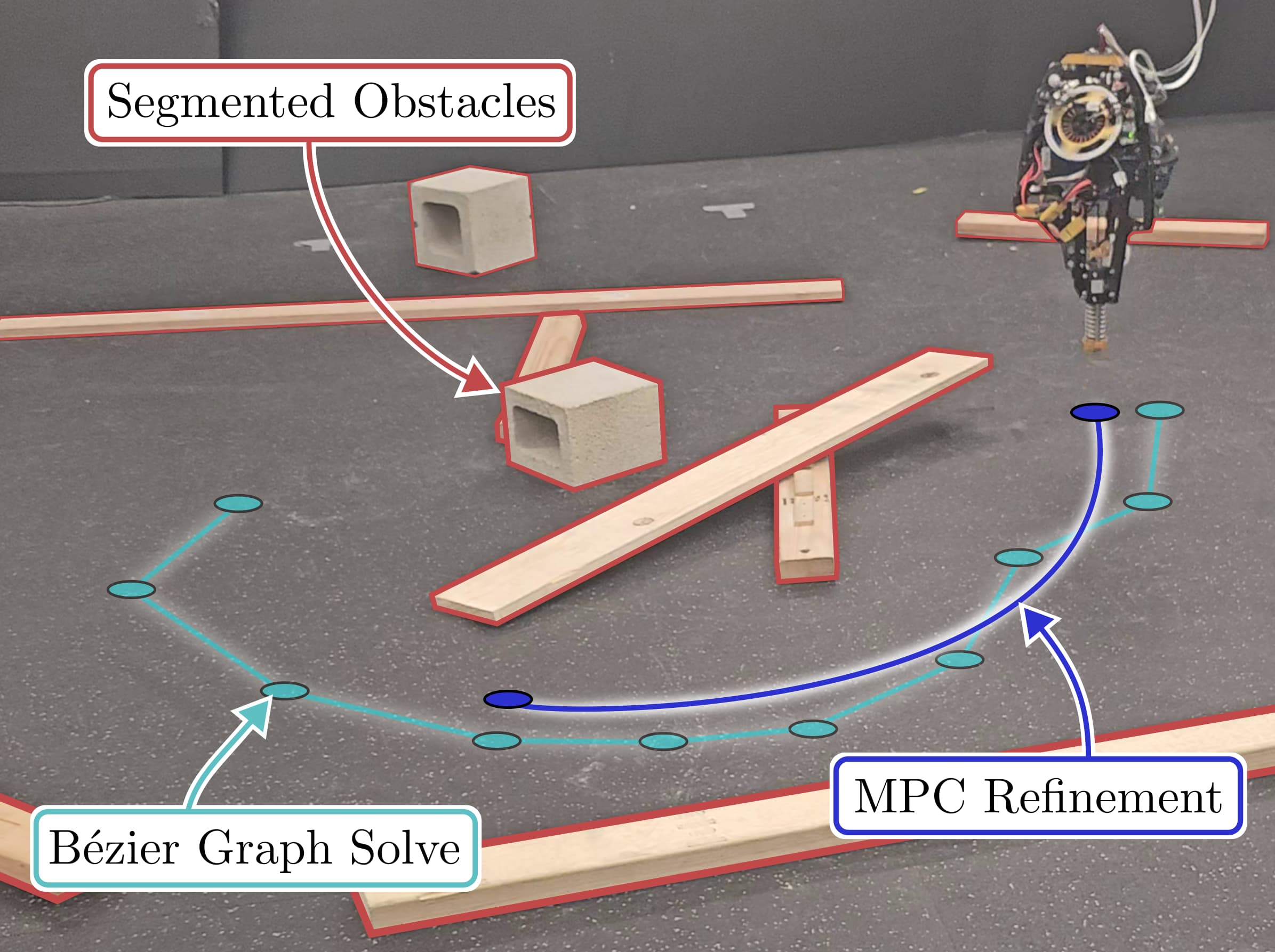}
    }
    \caption{The proposed framework performing path planning around segmented obstacles. A coarse dynamically feasible path is solved for on a graph with B\'ezier polynomial edges, and is further refined with MPC.}
    \label{fig:hero}
    \vspace{-6mm}
\end{figure}
% this suffers from potentially complex graphs, as the discrete collection of chosen primitives may not aid in getting to the goal.

% Reachability for path planning: Joel's citation

%
Sampling-based methods face two major the challenges -- path suboptimality due to the probability of sampling the true optimal path being zero, and bias, either through the discretized policy design, or Voronoi bias introduced by using a Euclidean metric \cite{voronoi}. To address suboptimality, various methods have proposed refinements to the path which increase the optimality \cite{webb_kinodynamic_2013}. To address the notion of bias, the notion of reachability has been introduced into sampling-based planners \cite{reeds_optimal_1990}.
In particular, authors have developed reachability-guided RRT variants \cite{fridovich-keil_planning_2018,shkolnik_reachability-guided_2009,wu_r3t_2020} that avoid solving boundary value problems during graph construction. For nonlinear systems, however, this top-down approach provides guarantees of feasibility which are often approximate and rely on using local linearizations of the nonlinearities; this requires the system to remain sufficiently close to the nominal trajectory during execution. 

In contrast to 
% sampling-based approaches, 
building a graph of sampled states, optimization-based approaches instead try to directly generate optimal paths to the goal state. These methods often require the solution of large mixed integer programs to generate optimal paths to the goal, either via a keep-out philosophy \cite{bigM}, or a keep-in philosophy via convex decomposition \cite{deits_efficient_2015, marcucci2021shortest, marcucci_motion_2022}.
% agrawal_constructive_2022}.
% \red{incorporate these:} \cite{akin_computing_2015,deits_efficient_2015}. 
To address the benefits and drawbacks sampling-based and optimization techniques, various approaches have been proposed to combine these methods \cite{lee2021efficient, stoneman2014embedding, fernbach2017kinodynamic}. Despite these advancements, the computational burden associated with these approaches hinders their deployment in obstacle-rich environments in real time.  As a means to mitigate these computational limitations, layered approaches to control synthesis have proven effective \cite{matni2024quantitative}.
%

% Efficiently computing kinodynamically feasible trajectories for nonlinear systems with complex constraints remains a challenge.  As demonstrated in the literature, identifying kinodynamically feasible state trajectories is critical, particularly in achieving this goal in a short time frame while maintaining trajectories that are globally optimal.

In this work, we leverage a layered control architecture for producing kinodynamically feasible trajectories for nonlinear systems. Motivated by \cite{marcucci_fast_2023}, we employ B\'ezier basis polyonmials for the task of solving paths through nonconvex spaces. In contrast to this and other works, significant effort is put towards guaranteeing dynamic feasibility of the generated paths in order to produce a guaranteed feasible pipeline. Specifically, we leverage the notion of \textit{B\'ezier reachable polytopes}, i.e., polytopic reachable sets in the space of polynomial B\'ezier reference trajectories, as a primitive for feasible trajectory tracking. Importantly, as the dynamics are abstracted via the reachable set of the mid level controller, the planner does not need to include differential constraints. This work builds on \cite{csomay-shanklin_multi-rate_2022} by incorporating a high level path planner in coordination with a mid and low level controller. As seen in Figure~\ref{fig:hero}, we demonstrate the effectiveness of this framework by planning paths in cluttered environments on the high-dimensional 3D hopping robot, ARCHER \cite{ambrose_creating_2022}.

% Canonically, desired trajectories would be planned and passed down to the tracking controller, which may lead to infeasible plans being produced for the underlying system. However, if the planning layer maintains information about the capability of the tracking controller, then only feasible plans can be generated. The aim of this paper is to produce trajectories, which when tracked result in state and input constraint satisfaction of the underlying robotic system.

\begin{figure*}
    \centering
    \includegraphics[width=\textwidth]{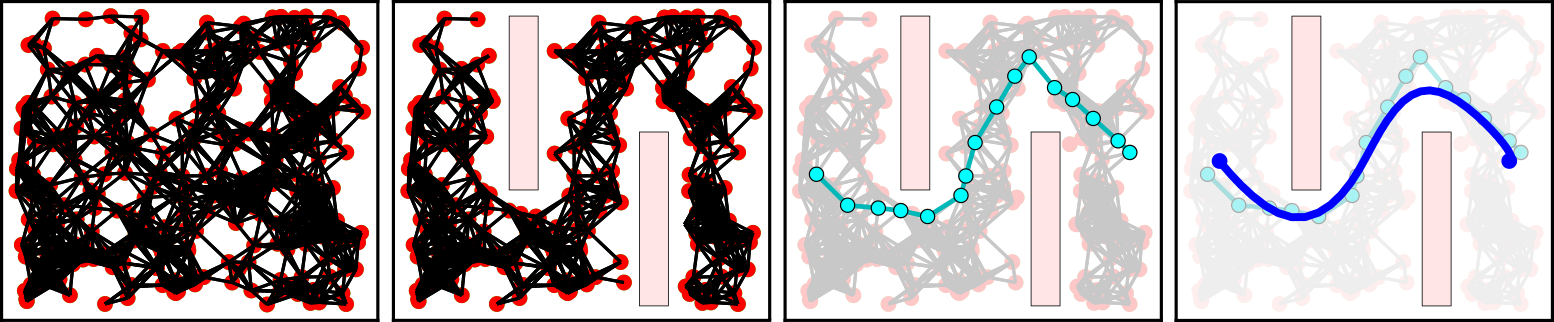}
    \vspace{-4mm}
    \caption{The path planning framework presented in Algorithm~\ref{algo:BezGraph}. From left to right: a) A B\'ezier graph is constructed, b) it is cut based on the present obstacles, c) a path is solved, and d) the path is refined with MPC.}
    \label{fig:main_algo}
    \vspace{-4mm}
\end{figure*}

\section{Background}
\subsection{Reduced Order Models and Problem Statement}
We begin with the following nonlinear control system:
\begin{align}
    \label{eqn:nonl}
	\dot{\b x} = \b f(\b x, \b u),
\end{align}
where $\b x \in \R^N$ is the state, $\b u \in \R^M$ is the input, and the vector $\b f: \R^N\times\R^M\to\R^N$ is assumed to be continuously differentiable in each argument. 
Directly synthesizing control actions that achieve complex tasks such as path planning may be challenging, in part due to the nonlinearities of the dynamics. To alleviate this complexity, roboticists often leverage reduced-order models, which serve as template systems that enable desired behaviors to be constructed in a computationally tractable way. To this end, consider the planning dynamics:
\begin{align}
\label{eqn:rom}
    \dot{\b x}_d  =\begin{bmatrix} \b 0 & \b I_{n-m} \\ \b 0&\b 0\end{bmatrix} \b x_d + \begin{bmatrix} \b 0 \\ \b I \end{bmatrix}\b u_d,
\end{align}
where $\b x_d \in \R^n$ is the reduced-order state and $\b u_d \in \R^m$ is the reduced-order input. Note that the the subsequent discussion, the planning dynamics in \ref{eqn:rom} can be extended to the case of full-state feedback linearizeable systems (as in \cite{csomay-shanklin_multi-rate_2022}), but for the sake of simplicity we restrict our attention to integrators. 

Let $I=\begin{bmatrix} 0 & T\end{bmatrix}\subset\R_{\ge 0}$ be a time interval, $\b u_d(\cdot)$ an input trajectory defined over I, and $\b x_d:I\to\R^n$ be an integral curve of the reduced-order model \eqref{eqn:rom}. We can correspond a reduced order system with a full order system through a surjective mapping $\b \Pi:\R^N\to \R^n$, and feedback controller $\b k:\R^N\times \R^n \to \R^M$, which takes in reduced order model trajectories and produces control actions for the full order system. A desired property of this controller is its ability to maintain bounded tracking error:
\begin{definition}\label{def:tracking_inv}
     The set $\mathcal{E}:I\to \mathcal{X}$ is a \textit{tracking invariant} for a desired trajectory $\b x_d : I \to \mathcal{X}_d$ for the system \eqref{eqn:nonl} if:
     \begin{align*}
        \b\Pi(\b x(t)) \in \b x_d(t) \oplus \mathcal{E}(t),
    \end{align*}
    where $\oplus$ denotes the Minkowski sum.
\end{definition}

Many existing works discuss ways to produce controllers with tracking invariant sets $\mathcal{E}$ \cite{apgar2018fast,sontag_characterizations_1995}; here, we assume such a tracking controller is given.
%
% \begin{assumption}
%     The tracking controller $\b k$ and tracking invariant set $\mathcal{E}$ are given.
% \end{assumption}
%
For a desired trajectory of the reduced-order model $\b x_d(\cdot)$, a suitable tracking controller $\b k$, and an initial condition $\b x(0)\in\R^N$, we can define the closed-loop system as:
\begin{align}
    \label{eq:cl}
    \b x_{\rm cl}(t)=\b x(0) + \int_0^t \b f(\b x(\tau), \b k(\b x(\tau), \b x_d(\tau)))d\tau.
\end{align}
We can now introduce the main problem statement:
\begin{problem}\label{prb:nav}
    Given a reduced order model \eqref{eqn:rom}, with state constraint $\cal{X}_d\subset \R^n$, a full order model \eqref{eqn:nonl} with input bounds $\cal{U}\subset \R^M$, a worse case tracking error bound $\bar{\cal{E}} \supset \cal{E}(t) \; \forall t$, and an initial and goal location for the full order model $\b x(0)$ and $\b x_G$,
    % satisfying $\Pi(\b x_0) \in \cal{X}_d \ominus \bar{\mathcal{E}}$, 
    % and a goal state for the reduced order model $\b x_{d,g}$,
    % $\in \cal{X}_d \ominus \bar{\cal{E}}$, 
    find a trajectory $\b x_d(t)$ such that:
    \begin{subequations}
     \begin{align}
        \lim_{t \to \infty} \b \Pi(\b x_{\rm cl}(t)) &\in \b\Pi(\b x_G) \oplus \bar{\cal{E}},\label{eqn:prb1_goal}\\
        \b \Pi(\b x_{\rm cl}(t)) &\in \cal{X}_d~~~\forall t, \label{eqn:prb1_safe} \\
        \b k(\b x_{\rm cl}(t), \b x_d(t)) &\in \cal{U} ~~~\forall t.\label{eqn:prb1_input}
    \end{align}   
    \end{subequations}
\end{problem}
Simply put, we are searching for a reduced order model trajectory such that the closed-loop full-order model will navigate to a goal position while respecting state and input constraints. To solve this problem, we will use B\'ezier curves to parameterize the space of trajectories passed from the planning level to the feedback level.

\subsection{B\'ezier Curves and Graphs}

A curve $\b b:I \to \R^m$ is said to be a B\'ezier curve \cite{kamermans2020primer} of order $\order\in\mathbb{N}$ if it is of the form:
\begin{align*}
    \b b(t) =   \Bez\b z(t),
\end{align*}
where $\b z:I \to \R^{\order+1}$ is a Bernstein basis polynomial of degree $\order$ and $\Bez \in \R^{m\times\order+1}$ are a collection of $\order+1$ \textit{control points} of dimension $m$. There exists a matrix $\b H\in\R^{p+1\times p+1}$  which defines a linear relationship between control points of a curve $\b b$ and its derivative via:
$$\dot{\b b}(t) =  \b \Bez \b H \b z(t).$$
This enables us to define a state space curve $\b B:I\to \R^{n}$:
\begin{align}
\label{eqn:BezDef}
	\b x_d(t) \triangleq \begin{bmatrix} \b b(t) \\  \vdots \\\b b^{(\gamma-1)}(t)\end{bmatrix} =  \underbrace{%
			\begin{bmatrix}\Bez \\ \vdots \\ \Bez\b H^{\gamma-1}\end{bmatrix}}_{\triangleq \bs{\BEZ}}\b z(t),
\end{align}
% \begin{align}
% \label{eqn:BezDef}
% 	\b x_d(t) \triangleq \begin{bmatrix} \b b(t) \\  \dot{\b b}(t) \\ \vdots \\\b b^{(\gamma-1)}(t)\end{bmatrix} =  \underbrace{%
% 			\begin{bmatrix}\Bez \\ \Bez\b H \\ \vdots \\ \Bez\b H^{\gamma-1}\end{bmatrix}}_{\triangleq \bs{\BEZ}}\b z(t).
% \end{align}
where $\gamma > 0$ is the relative degree of an output $\b b(t)$ for the system \eqref{eqn:rom} and is defined such that $m\gamma = n$.
The columns of the matrix $\BEZ\in\R^{n\times p+1}$, denoted as $\b \BEZ_{j}$ for $j=0,\ldots,p$, represent the collection of $n$ dimensional control points of the B\'{e}zier curve $\b B(\cdot)$ in the state space.

Given a B\'ezier polynomial control input $\b u_d(\cdot)$, the integral curve $\b x_d(\cdot)$ can be exactly described by the B\'ezier polynomial \eqref{eqn:BezDef} of degree $\gamma$ larger than that of $\b u_d(\cdot)$. Furthmore, given two boundary conditions, there exist unique B\'ezier curves (of minimal degree $p = 2\gamma-1$) connecting them:
\begin{lemma}[\cite{future_acc}]
\label{lemm:disc_to_bez}
    Given two points $\b x_0, \b x_T \in \R^n$, there exists a unique matrix $\b D$ such that $\b \BEZ = \b D \begin{bmatrix} \b x_0 & \b x_T\end{bmatrix}$ are the control points of a B\'ezier curve $\b x_d(\cdot)$ satisfying $\b x_d(0) = \b x_0$ and $\b x_d(T) = \b x_T$.
\end{lemma}

This statement allows us to plan dynamically feasible curves simply by reasoning about discrete points in state space.
This concept is paried with a key property of B\'ezier curves -- namely, the continuous time curve is bounded by the convex hull of the control points:
\begin{property}[Convex Hull \cite{kamermans2020primer}]\label{prp:conv_hull}
A B\'ezier curve $\b x_d$ is contained in the convex hull of its control points:
	$$\b B(t) \in \text{conv}(\{\b \BEZ_{j}\}),~~ j=0,\ldots, \order,~~\forall t\in I.$$
\end{property}
% This property allows us to make guarantees of the continuous time curve while only reasoning about a discrete, low-dimension set of control points. For a complete construction of the matrices discussed in this section, see the \cite{future_acc}.
\begin{property}[Path Length \cite{marcucci_fast_2023}]\label{prp:path_length}
    The path length of a B\'ezier curve $\b x_d(\cdot)$ is upper bounded by the norm distance of its control points:
    \begin{align*}
        \int_0^T \| \dot{\b x}_d(\tau)\|_2 d\tau \le \sum_{j=0}^{p-1} \|\b \BEZ_{j+1} - \b \BEZ_{j}\|_2
    \end{align*}
\end{property}
\begin{property}[Subdividion \cite{kamermans2020primer}]\label{prp:subd}
    A B\'ezier curve $\b x_d$ defined on the time interval $I=\begin{bmatrix} 0 & T\end{bmatrix}$ can be subdivided into arbitrarily many B\'ezier curves given a subdivision of $I$.
\end{property}

Besides the standard aforementioned properties, the reachable set of B\'ezier polynomials which when tracked satisfy state and input constraints of the tracking system is efficiently representable:
\begin{theorem}[\cite{future_acc}]
\label{thm:bez}
    Given a convex state constraint set $\mathcal{X}_d\subset\R^n$, input constraint set $\mathcal{U}\subset\R^M$ and error tracking bound $\bar{\mathcal{E}}$, there exist matrices $\b F$ and $\b G$ such that any two points $\b x_1,\b x_2\in\R^n$ satisfying:
    \begin{align*}
        \b F\begin{bmatrix} \b x_1^\top  & \b x_2^\top\end{bmatrix}^\top \le \b G
    \end{align*}
    implies the existence of a B\'ezier curve $\b x_d$ with $\b x_d(0) = \b x_1$, $\b x_d(T) = \b x_2$, such that, when tracked, the closed loop system satisfies $\b \Pi( \b x_{\rm cl}(t)) \in \mathcal{X}_d$ and $\b k(\b x_{\rm cl}(t) ,\b x_d(t)) \in \mathcal{U}$.
\end{theorem}

The matrices $\b F$ and $\b G$ can be thought of as reachable set oracles, as they determine if there exists dynamically feasible curves between two points. This notion of B\'ezier reachable sets will be used to solve Problem 1 by first efficiently generating a graph of dynamically feasible Be\'zier curves.

% \subsection{Graphs}
A graph $\mathcal{G}$ is described by a tuple $(V, E)$ with vertices $V$ connected by edges $E$. 
For the purpose of our analysis, the vertices will represent points in the reduced-order state space $\R^n$. A directed edge between two vertices $e=(v_1,v_2)$ for $v_1,v_2\in V$ will represent the existence of a dynamically feasible trajectory connecting $v_1$ to $v_2$.
% For the sake of this work, we will produce edges between points precisely when $\b F\begin{bmatrix} \b v_1^\top  & \b v_2^\top\end{bmatrix}^\top \le \b G$. 
% 

\section{Kinodynamic B\'ezier Graphs}

{\setlength{\textfloatsep}{3pt}
\begin{algorithm}[b!]
\caption{B\'ezier Graphs}
\label{algo:BezGraph}
\begin{small}
\begin{algorithmic}[1]
\State \textbf{hyperparameters:} $(N, T)$ \Comment{(Node count and time interval)}
\State $\mathcal G\leftarrow$ buildGraph($\mathcal{X}_d$, $\mathcal{U}$)
\State ${\mathcal{C}}\leftarrow$ cutGraph($\mathcal{G}$, $\mathcal{O}$)
\State $\b v^*_k\leftarrow$findPath$(\mathcal{C}$)
\State $\b x_d(\cdot)\leftarrow$refineWithMpc($\b v^*_k, \mathcal{O}$)
\State \Return $\b x_d(\cdot)$
\end{algorithmic}
\end{small}
 \label{alg:train}
\end{algorithm}
% \begin{figure}
% \vspace{-20pt}
% \end{figure}
}

The goal of this section will be to produce desired trajectories $\b x_d(\cdot)$ which satisfy Problem~\ref{prb:nav} when obstacles are present, i.e. the free space is non-convex.
% which, when tracked, result in state and input constraint satisfaction of the nonlinear system \eqref{eqn:nonl}. 
The approach will be to produce a graph of dynamically feasible B\'ezier curves leveraging the \textit{B\'ezier reachable polytopes} introduced in Theorem~\ref{thm:bez}, cut the edges that intersect obstacles, and perform a graph search that will serve as a feasible warm start for a trajectory optimization program. The paradigm is outlined in Algorithm~\eqref{algo:BezGraph} and shown in Figure~\ref{fig:main_algo}, and starts with building a graph of B\'ezier curves.

\subsection{Build Graph}
The compact state space of interest $\mathcal{X}_d$ represents the free space of the problem setup, and will serve as a seed which we can prune in the presence of obstacles.
% We start by fixing a compact state and input space of interest: $\mathcal{X}_d\subset\R^n$ and $\mathcal{U}\subset\R^M$, i.e. a state constraint of the reduced order model and an input constraint of the full order model.
% Here, the state space
For the sake of limiting algorithmic complexity, we make the following assumption about the description of the free space:
\begin{assumption}
The free space is given by $\mathcal{X}_d\backslash\mathcal{O}\subset\R^n$ for a set of obstacles $\mathcal{O} = \cup_i\mathcal{O}_i$ where $\mathcal{O}_i$ is convex polytope.
\end{assumption}
To begin, $N$ points are uniformly sampled $\b v_i \sim \rm U(\mathcal{X}_d)$ and the matrices $\b F$ and $\b G$ from Theorem~\ref{thm:bez} are generated. Two points $\b v_i$ and $\b v_j$ are connected with an edge $\b e_{i, j}$ if they satisfy $\b F \begin{bmatrix} \b v_i^\top & \b v_j^\top \end{bmatrix}^\top \le \b G$.
As such, every edge in the graph implies the existence of a dynamically feasible B\'ezier curve connecting them. This initial graph will be fixed throughout the remainder of the Algorithm.

{\setlength{\textfloatsep}{3pt}
% \begin{figure}
% \vspace{-2pt}
% \end{figure}
\begin{algorithm}[t!]
\caption{cutHeuristic}
\label{algo:heuristic}
\begin{small}
\begin{algorithmic}[1]
\If{any($\b A \b v_i \le \b b)$} \Comment{Figure~\ref{fig:heuristic}a)}
\State return 0
\EndIf
\State $\b v^o\leftarrow$closestPoint($\b v_i)$
\State $(\b h, \b l)\leftarrow$adjacentHyperplane($\b v^o$)
\If{all$(\b h \b v_i \le \b l)$} \Comment{Figure~\ref{fig:heuristic}b)}
\State return 1
\EndIf
\State Solve QP \Comment{Figure~\ref{fig:heuristic}c)}
\State return $\|\b\delta^*\| > 0$
\end{algorithmic}
\end{small}
 \label{alg:cutHeuristic}
\end{algorithm}
% \begin{figure}
% \vspace{-10pt}
% \end{figure}
}

\subsection{Cut Graph}
Given a graph of B\'ezier curves $\mathcal{G}$, we remove the edges that could result in collisions with any obstacles that are present. Given a set of obstacles $\mathcal{O}$, Property~\ref{prp:conv_hull} states that a B\'ezier curve is guaranteed to be collision-free if the convex hull of its control points does not intersect with any obstacles, i.e. $\text{conv}\{\b \BEZ\} \cap \mathcal{O} = \emptyset$. For each edge $\b e$ with control points $\b \BEZ$ and convex obstacle $\mathcal{O}_i$ characterized by $n_c$ hyperplane constraints $\b A^{\mathcal{O}_i}\in\R^{n\times n_c}$ and $\b b^{\mathcal{O}_i}\in \R^{n_c}$, the problem of determining if the curve is collision-free can be formulated as a linear program feasibility check via Property~\eqref{prp:conv_hull}, which we solve using the following quadratic program:
\begin{subequations}
\begin{align}
\min_{\b \lambda, \b\delta} \quad & \b \delta^\top \b\delta \label{eqn:feas_check} \tag{Cut-QP}\\
\textrm{s.t.} \quad & \b A^{\mathcal{O}_i}\b\BEZ \b\lambda \le \b b^{\mathcal{O}_i} + \b\delta \notag\\
  & \b \lambda_j \ge 0,~~~~~\sum_j\b\lambda_j = 1 \notag
\end{align}
\end{subequations}
where $\b\lambda\in\R^{p+1}$ is the convex interpolation variable and $\b \delta\in\R^{n_c}$ is a slack variable. Given the solution $\b\delta^*$, if $\|\b\delta^*\|_2 > 0$ then the curve is obstacle free. This program can easily be vectorized to solve all edges simultaneously.

Although this is a quadratic program and therefore can be solved efficiently, solving it for every edge and obstacle becomes too computationally expensive for real-time operation (even when parallelized), especially when the edge and and number of obstacle count is large. To address this, we introduce the heuristic outlined in Algorithm~\ref{alg:cutHeuristic}, which consists of three checks. First, we determine if any of the control points of $\b \BEZ$ lie in the obstacle -- if so, the curve may intersect the obstacle. Next, we attempt to find a separating hyperplane between the obstacle and the control points -- if one exists, then the curve will be collision free. If the points do not lie on one side of the hyperplane, then the heuristic is indeterminate, and the original quadratic program must be solved.  In practice, this heuristic typically removes > 99\% of the edges that need to be checked; furthermore, this heuristic is easily implemented on GPU as it only requires linear algebra operations, which significantly improves execution speed (as reported in Section~\ref{sec:results}).

\subsection{Find Path}
Due to the random sampling of the points in the graph, the starting and goal positions are likely not vertices. Therefore, the start and goal nodes are chosen as the closest normed-distance points to the desired start and goal positions. The edge cost is taken to be $\sum_i \|\b\BEZ_{i+1} - \b \BEZ_i\|$, the upper bound on the path length of the B\'ezier curve from Property \ref{prp:path_length}. We then use Dijkstra`s algorithm to solve Problem \ref{prb:graph}:

\begin{problem} \label{prb:graph}
    In the context of Problem \ref{prb:nav}, consider a collision-free graph of B\'ezier curves $\cal{C}=(V,E)$ produced by Algorithm~\ref{algo:BezGraph}.
    % satisfying $\b v \in \cal{X}_d \backslash \mathcal{O} \ominus \bar{\cal{E}}$ for all vertices $\b v \in V$.
    % , as well as $\b\Pi( x_G) \in V$ and the existence of $\b v_0 \in V$ such that $\Pi(\b x_0) \in \b v_0 \oplus \bar{\cal{E}}$. 
    Find a path $\{\b v_k^*\}_{k=0}^K$ in the graph connecting $\b v_0^*$ to $\b v^*_K = \b\Pi(x_G)$ with $\b v_0^*$ satisfying $\Pi(\b x(0)) \in \b v_0^* \oplus \bar{\cal{E}}$.
\end{problem}

\begin{figure}[t!]
    \centering
    \includegraphics[width=0.93\columnwidth]{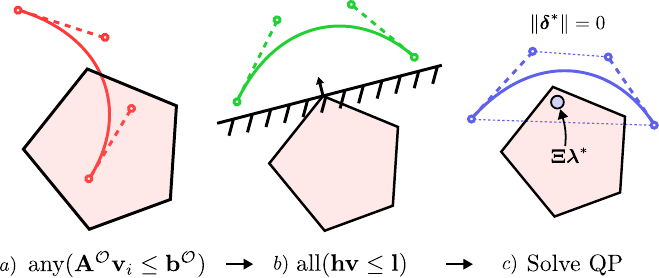}
    \caption{The heuristic employed to check if the B\'ezier curve is collision free. First, obstacle membership is checked. If not satisfied, a single proposed separating hyperplane is checked. Finally, the Quadratic Program \eqref{eqn:feas_check} is solved.}
    \vspace{-5mm}
    \label{fig:heuristic}
\end{figure}

\subsection{Refine the Path}
If the graph solve $\b v^*_k$ exists, it represents a feasible path for the system to safely traverse the cluttered environment; however, it may be significantly suboptimal. To improve its optimality, we can solve a finite time optimal control program, which balances tracking the graph solution with short-horizon optimality.
Note that by Lemma~\ref{lemm:disc_to_bez}, a solution Problem \ref{prb:graph} can be converted into a dynamically feasible curve $\b x_d(\cdot)$ for the reduced order model \eqref{eqn:rom}.
% To begin, we convert $\b v_k^*$ from a discrete collection of points to a desired trajectory $\b x_d(\cdot)$ via the minimal curve fitting matrix $\b D$. 
Then, we sample this curve at the optimal control time discretization to produce a sequence of reference points $\b r_k$, which can be optimized:
\begin{subequations}\label{eqn:mpc}
\begin{align}
\min_{\b \x, \b u} \quad & \sum_{k=0}^{N-1} \b x_k ^\top \b Q \b x_k + \b F \b r_k + \b u_k^\top \b R \b u_k + \b x_N ^\top \b V \b x_N \label{eqn:mpc_cost} \tag{MPC}\\
\textrm{s.t.} \quad & \b x_{k+1} = \b A \b x_k + \b B \b u_k  \label{eqn:mpc_dyn}\\
  & \b x_0 \in \b \Pi(\b x(0)) \oplus \bar{\cal{E}}  \label{eqn:mpc_ic}\\
  & \b D \begin{bmatrix}\b x_k & \b x_{k+1}\end{bmatrix} \in \mathcal{X}_d\backslash(\mathcal{O}\oplus \mathcal{E})  \label{eqn:mpc_safe}\\
  &\b F \hspace{.5mm}\begin{bmatrix}\b x_k & \b x_{k+1}\end{bmatrix} \leq \b G  \label{eqn:mpc_input}\\
  & \b x_N = \b r_N  \label{eqn:mpc_term}
\end{align}
\end{subequations}
where $\b Q, \b F\in \R^{n\times n}$ are symmetric positive definite matrices weighting distance to reference as well as path length, $\b R\in \R^{m\times m}$ is a positive definite input scaling matrix, and $\b V \in \R^{n\times n}$ is a terminal cost. The matrices $\b A\in\R^{n\times n}$ and $\b B\in\R^{n\times m}$ are the exact discretization of the integrator dynamics in \eqref{eqn:rom}.
The constraint $\mathcal{X}_d\backslash(\mathcal{O}\oplus \mathcal{E})$ is enforced by fining a separating hyperplane between the node $\b x_k$ and the closest obstacle (where the same adjacentHyperplane call form the heuristic can be made). In order to improve the quality of solutions, \eqref{eqn:mpc_cost} can be iteratively solved in an SQP fashion. 

\begin{figure}[t!]
    \centering
    \includegraphics[width=\linewidth]{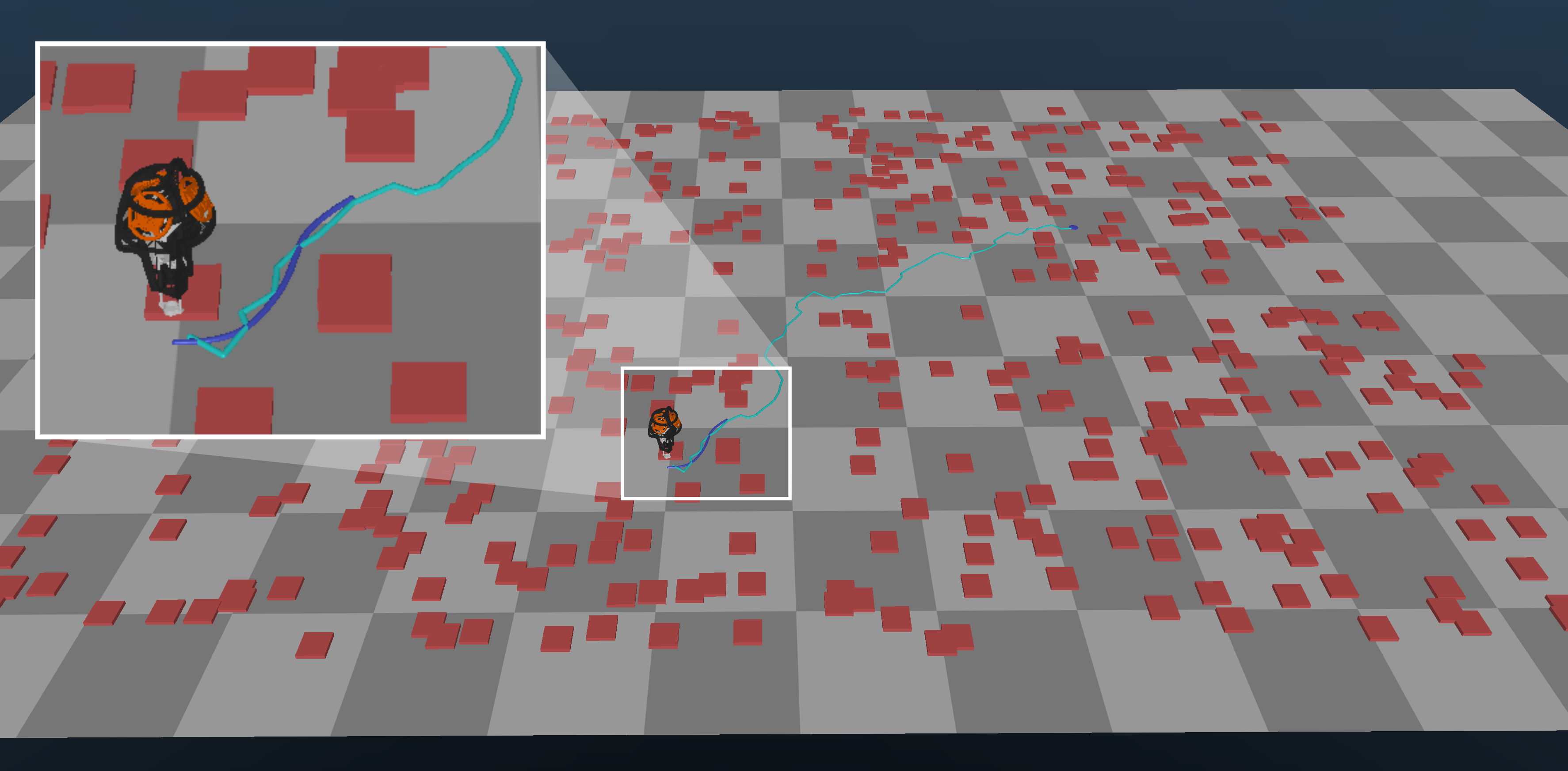}
    \vspace{-3mm}
    \caption{500 randomly generated obstacles with graph replanning at 10 Hz and \eqref{eqn:mpc_cost} at 200 Hz. Cyan indicates the B\'ezier graph solve $\b v_k^*$ and blue the MPC B\'ezier solution $\b x_d(\cdot)$.}
    \vspace{-5mm}
    \label{fig:mujoco_obst}
\end{figure}

\begin{table}[b!]
\centering
\caption{ \small Comparison of functions on CPU, parallelized CPU, and GPU. Any CPU implementation would violate real time constraints.}
\label{tab:timing}
\begin{tabular}{l||c|c}
& {cutHeuristic} & {adjacentHyperplane} \\ \hline \hline
CPU & 5920 $\pm$  36 ms & 6.24 $\pm$ 0.08 ms  \\ \hline
Parallel CPU & 3640 $\pm$  31 ms &  2.52 $\pm$ 0.33 ms \\ \hline
GPU & 72.0 $\pm$  1.5 ms & 0.61 $\pm$ 0.14 ms 
\end{tabular}
\end{table}

Importantly, the graph solve represents a feasible warm start for the MPC program, and can be thought of as providing a corridor and a dynamically feasible path which MPC can refine. The main motivation for the structures used in this paper is the notion of guaranteed feasibility:
\begin{theorem}
    % If a solution to Problem \ref{prb:graph} $\b v_k^*$ exists, running \eqref{eqn:mpc_cost} in closed loop produces a trajectory $\b x(\cdot )$ satisfying Problem \ref{prb:nav}.
    If Problem \ref{prb:graph} is feasible, then Problem \ref{prb:nav} is solved by applying \eqref{eqn:mpc_cost} in closed loop.
    % , i.e. robust stability and constraint satisfaction, is achieved.
\end{theorem}
\begin{proof}
    First, we must show that at time $t=0$, the solution $\b v_k^*$ from Problem \ref{prb:graph} provides a feasible solution for \eqref{eqn:mpc_cost}. Let $\b r_k$ denote the refined sequence of $\b v_k^*$, which is constructed to produce the reference for \eqref{eqn:mpc_cost} using Property~\ref{prp:subd}. As this is associated with a B\'ezier curve for the reduced-order model, it satisfies the discrete dynamics \eqref{eqn:mpc_dyn}. Next, by the definition of Problem \ref{prb:graph}, we know $\Pi(\b x_0) \in \b v_0^* \oplus \bar{\cal{E}}$, which implies that \eqref{eqn:mpc_ic} is satisfied by Mikoswki addition properties. As $\b v_k^*$ is in the collision-free graph $\mathcal{C}$, it satisfies \eqref{eqn:mpc_input} and \eqref{eqn:mpc_safe}. By construction, \eqref{eqn:mpc_term} is satisfied. Therefore, $\b r_k^*$ represents a feasible solution for \eqref{eqn:mpc_cost}.

    Next, take $\b x_k^*$ to denote the MPC solution and consider the time interval $I=\begin{bmatrix}0& h\end{bmatrix}$ for MPC discretization $h>0$. From Lemma~\eqref{lemm:disc_to_bez}, we know that there exists a unique curve $\b x_d(\cdot)$ defined over $I$ connecting $\b x_0^*$ to $\b x_1^*$. By Theorem~\ref{thm:bez}, the closed loop system tracking this curve will satisfy \eqref{eqn:prb1_safe} and \eqref{eqn:prb1_input} for all $t\in I$. Furthermore, by Definition~\eqref{def:tracking_inv}, we have that $\b\Pi(\b x(h))\in \b x_1^*\oplus\bar{\mathcal{E}}$. Therefore, we can appeal to standard Robust tube MPC theory \cite{mpc_book} to claim recursive feasibility and robust stability of the closed loop system tracking the MPC solution, meaning \eqref{eqn:prb1_goal} is achieved and \eqref{eqn:prb1_safe} and \eqref{eqn:prb1_input} are satisfied for all time.
\end{proof}
We have shown that if there is a solution to the graph problem in Algorithm~\ref{algo:BezGraph}, then the closed loop system tracking MPC will satisfy state and input constraints $\mathcal{X}_d\backslash \mathcal{O}$ and $\mathcal{U}$ for all time, and will approach a neighborhood of the goal $\b x_G$.

\section{Results}
\label{sec:results}

\begin{figure}[t!]
    \centering
    \includegraphics[width=\linewidth]{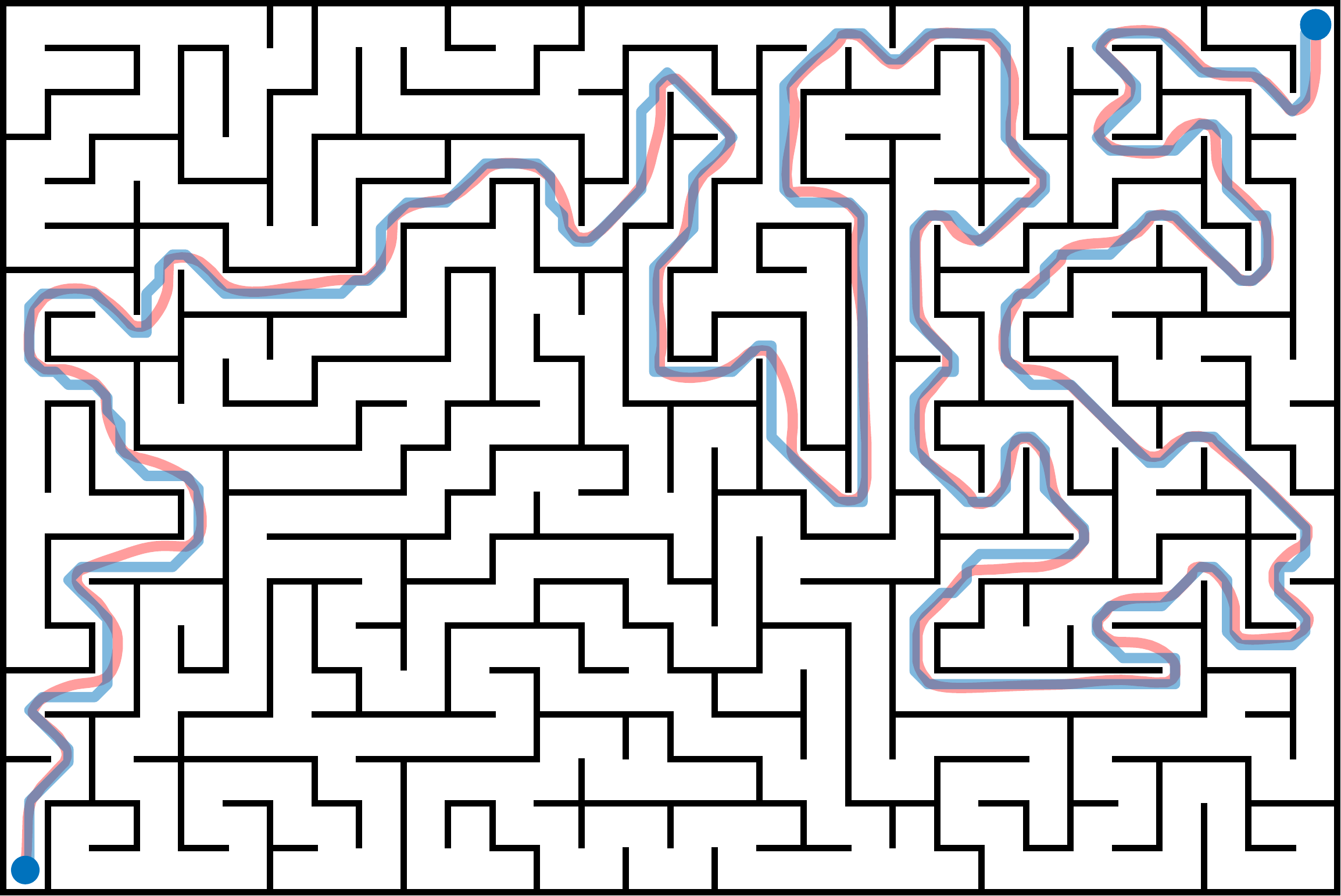}
    \vspace{-3mm}
    \caption{A long-horizon maze with 600 cells, 300 obstacles and 50,000 edges. The graph cut is solved at 10 Hz, and graph solve and MPC both run at 200 Hz. Blue represents the graph solve, and red the closed loop system behavior.}
    \vspace{-4.5mm}
    \label{fig:maze}
\end{figure}

\begin{figure*}
    \centering
    \includegraphics[width=\linewidth]{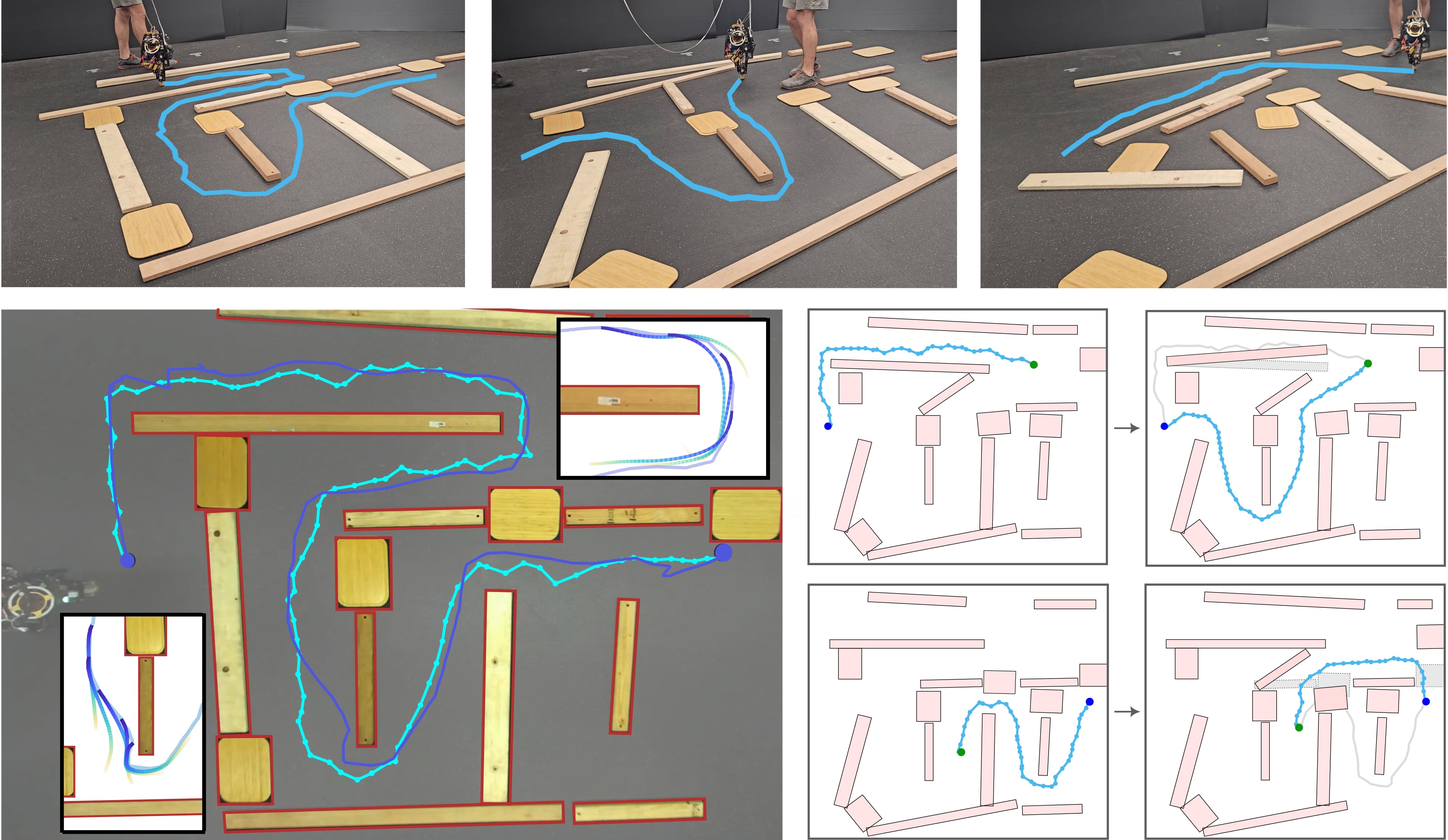}
    \vspace{-3mm}
    \caption{Experiments run on the ARCHER hardware Platform. (Top) 3 snapshots of the graph solve for various obstacle configurations. (Left) An overhead image of the graph solve and the closed loop system trajectory. In the upper right and bottom left corner, the MPC refinement around the corner can be seen. (Right) Freeze frames of the graph updating online when the obstacle locations are moved.}
    \vspace{-4.5mm}
    \label{fig:hardware}
\end{figure*}

We deploy Algorithm \ref{algo:BezGraph} on the 3D hopping robot, ARCHER \cite{ambrose_creating_2022} -- a video can be found at \cite{supplemental_video} and the code for this project is available at \cite{code}. Let $(\b p, q)\in\R^3\times \mathbb{S}^3$ denote the global position and quaternion of the robot, and $(\b v, \b \omega)\in\R^3 \times \mathfrak{s}^3$ the global linear velocity and body frame angular rates. The full state of the robot $\b x\in\mathcal{X}\subset\R^{20}$ contains all of the above states, as well as foot and flywheel positions and velocities. Although planning dynamically feasible paths on the full hybrid system dynamics is possible \cite{csomay2023nonlinear}, the added complexity of planning collision free paths on the full order state directly quickly becomes intractable. The most challenging problem preventing the nonconvex problem from being directly solved in real time is appropriately choosing a collection of corridors to traverse. By combining problems at various time scales, the graph solve is able to provide a coarse estimate of the keep in constraints while foregoing path optimality, and the short horizon MPC is able to refine it into a more optimal path. Therefore, we choose to plan desired center of mass trajectories $\b x_d \in \R^4$ with virtual force inputs $\b u_d\in\R^2$, which are assumed to have double-integrator dynamics.

The projection map $\b \Pi:\mathcal{X} \to \R^4$ is taken to be the restriction of the full order state to the center of mass $x$ and $y$ positions and velocities. The tracking controller $\b k$ is a Raibert-style controller, which takes in desired center of mass state and input trajectories and produces desired orientation quaternions as:
\begin{align*}
    q_d(\b x, t) = \b K_{\rm fb} (\b \Pi(x) - \b x_d(t)) + \b K_{\rm ff}\b u_d(t)
\end{align*}
and desired angular rates $\b \omega_d \equiv \b 0$. This desired quaternion is then tracked by a low-level controller:
\begin{align*}
    \b u(\x, t) = -\b K_{\rm p}  \mathbb{I}\textrm{m}(q_{ d}(t)^{-1} q) - \b K_{\rm d}(\b\omega - \b\omega_{ d}(t)),
\end{align*}
which runs at 1 kHz. For all experiments, the layered controller was running on an Ubuntu 20.04 machine with an AMD Ryzen 5950x @ 3.4 GHz, an NVIDIA GeForce RTX 4090, and 64 Gb RAM. 

In order to solve \eqref{eqn:feas_check} and \eqref{eqn:mpc_cost}, the OSQP library \cite{osqp} was used with maximum iterations limited to 50. The heuristic for cutting the graph and getting the separating hyperplanes was running on a custom written CUDA kernel. Using the GPU to perform these operations was critical to real-time performance -- as summarized in Table~\ref{tab:timing}, the CUDA kernel provided a significant speedup as compared to multithreaded CPU implementation. 
A nominal graph with 0.5 second B\'ezier curves and a 50 node MPC horizon with a timestep of 0.1 s were used. Each iteration of \eqref{eqn:mpc_cost} was solved at 100 Hz, and in practice 1 SQP iteration was taken. Finally, path length cost was scaled significantly higher than tracking cost in order to incentivize shorter paths than the graph solve could provide. 

\subsection{Simulation and Hardware}
As can be seen in Figure~\ref{fig:mujoco_obst} and Figure~\ref{fig:maze}, the proposed framework is able to solve extremely long horizon tasks with numerous obstacles in real time. Figure~\ref{fig:mujoco_obst} shows how the MPC solution $\b x_d(\cdot)$ can refine the graph solve $\b v_k^*$ to improve the optimality of the path while maintaining collision avoidance. In Figure~\ref{fig:maze}, a grid was used instead of random sampling due to the structured nature of the problem. In both of these settings, about 5000 nodes were sampled for the graph, leading to 50,000 edges to have to be processed every graph solve time step. In this setup, the \eqref{eqn:feas_check} takes 2.7 seconds to run per obstacle, even when limited to 50 solver iterations. When the heuristic runs, it is able to remove > 99\% of the edges and leads to solve times of 10-20 ms.

For the hardware setup seen in Figure~\ref{fig:hardware}, in order to estimate obstacle locations, we had an overhead ZED 2 camera and used the SAM2 segmenting repository \cite{kirillov2023segany}. Once the experiment started, the segmenter was initialized with a single click per obstacle, and was able to parse obstacle locations for the remainder of the experiment, which were streamed over ROS to the hierarchical controller. 
SAM2 segmented 20 obstacles at 4 Hz, and the graph cut and MPC ran at 50 Hz and 200 Hz, respectively.
We provided various cluttered environments to ARCHER, which was able to solve for feasible paths and traverse them in real time. In the bottom right of Figure~\ref{fig:hardware}, the real-time replanning of the graph can be seen as the obstacles were kicked around in the environment.

% \begin{figure}
%     \centering
%     \includegraphics[width=\linewidth]{Figures/PathPlanHardware.png}
%     \caption{Caption}
%     \label{fig:enter-label}
% \end{figure}

% \begin{figure}
%     \centering
%     \includegraphics[width=\linewidth]{Figures/HardwareReplan.png}
%     \caption{Caption}
%     \label{fig:enter-label}
% \end{figure}

% \begin{figure*}
%     \centering
%     \includegraphics[width=\linewidth]{Figures/ExperimentalBanner.pdf}
%     \caption{Caption}
%     \label{fig:enter-label}
% \end{figure*}

\section{Conclusion}
We proposed a layered control architecture for performing path planning through cluttered environments. By leveraging the properties of B\'ezier curves, a kinodynamically feasible graph for a planning model could be efficiently constructed, and could provide curves that were guaranteed to satisfy state and input constraints of the full order system when tracked.
By leveraging the CPU and the GPU together, Algorithm~\eqref{algo:BezGraph} could efficiently be run in real time for extremely long-horizon tasks such as maze solving both in simulation and on hardware.
Future work includes finding automated ways of determining the graph size and Be\'zier horizon length and using adaptive time scales for combining high performance and high precision tasks.
% , and applying this framework to other systems.
% and use flow-barrier functions to move the nodes of the graph instead of cut them, thereby maintaining conenctivity in tight spaces.

\bibliographystyle{IEEEtran}
\balance
\bibliography{main.bib}
\end{document}

%% file: preamble.tex
\usepackage{times}
\usepackage{setspace}
\usepackage{url}
\usepackage[utf8]{inputenc}
\usepackage[T1]{fontenc}
\usepackage{graphicx}		% For \includegraphics
\usepackage[font=small]{caption}
\usepackage{amsmath} % assumes amsmath package installed
\usepackage{amssymb}  % assumes amsmath package installed
\usepackage{amsthm}
\usepackage{mathtools}
\usepackage[normalem]{ulem}
\usepackage{paralist}	% For enumerations with better titled numbers
\usepackage[space]{grffile} % Space in filenames
\usepackage[dvipsnames,usenames]{xcolor}
\usepackage{bbm}
\usepackage{placeins} % Floatbarrier
\usepackage{array}
\usepackage{siunitx}
\usepackage{hyperref}
\usepackage{cleveref}
\usepackage{bm} % bold in math env
\usepackage{algorithm}
\usepackage{algpseudocode}
\usepackage{cleveref}
\usepackage{cite}

\newtheorem{theorem}{Theorem}

\newtheorem{lemma}{Lemma}
\theoremstyle{definition}
\newtheorem{definition}{Definition}
\theoremstyle{definition}

\theoremstyle{remark}

\theoremstyle{definition}
\newtheorem{property}{Property}
\theoremstyle{definition}
\newtheorem{assumption}{Assumption}
\theoremstyle{definition}

\theoremstyle{definition}
\newtheorem{problem}{Problem}
\theoremstyle{definition}

\theoremstyle{definition}

\theoremstyle{definition}

\theoremstyle{definition}

\renewcommand{\cal}[1]{\mathcal{ #1 }}
\newcommand{\mb}[1]{\mathbf{ #1 }}
\newcommand{\bs}[1]{\boldsymbol{#1}}
\renewcommand\b[1]{%
  \ifcat\noexpand#1\relax % check if the argument is a control sequence
    \bm{#1}% probably Greek
  \else
    \mathbf{#1}% single character
  \fi
}

\newcommand{\order}{p}

\newcommand{\R}{\mathbb{R}}

\newcommand{\Bez}{\b p}
\newcommand{\BEZ}{\b P}
\newcommand{\x}{\mb{x}} % state